\documentclass{article}
\usepackage{iclr2026_conference,times}

\usepackage{amsmath,amsfonts,bm}

\def\eqref#1{equation~\ref{#1}}

\def\1{\bm{1}}

\DeclareMathAlphabet{\mathsfit}{\encodingdefault}{\sfdefault}{m}{sl}
\SetMathAlphabet{\mathsfit}{bold}{\encodingdefault}{\sfdefault}{bx}{n}

\newcommand{\Var}{\mathrm{Var}}

\newcommand{\Cov}{\mathrm{Cov}}

\usepackage{hyperref}
\hypersetup{
  colorlinks=true,
  linkcolor=blue!45!black,
  citecolor=blue!45!black,
  urlcolor=blue!45!black
}
\usepackage{url}
\usepackage{amsthm,booktabs,graphicx,float,microtype,tikz,etoolbox}
\usetikzlibrary{arrows.meta,positioning}
\definecolor{ink}{HTML}{1C2833}
\definecolor{lane}{HTML}{EAF2F8}
\definecolor{lane2}{HTML}{FDEDEC}
\definecolor{accent}{HTML}{1A5276}
\definecolor{good}{HTML}{196F3D}
\definecolor{warn}{HTML}{7B241C}

\iclrfinalcopy
\makeatletter
\patchcmd{\@maketitle}{\rule{\z@}{24pt}}{\rule{\z@}{6pt}}{}{}
\patchcmd{\@maketitle}{\rule{\z@}{24pt}}{\rule{\z@}{6pt}}{}{}
\patchcmd{\@maketitle}{\rule{\z@}{24pt}}{\rule{\z@}{6pt}}{}{}
\patchcmd{\@maketitle}{\vskip 0.3in minus 0.1in}{\vskip 0.14in}{}{}
\makeatother

\newtheorem{theorem}{Theorem}
\newtheorem{proposition}[theorem]{Proposition}
\newtheorem{corollary}[theorem]{Corollary}

\theoremstyle{definition}

\newcommand{\Ex}{\mathbb{E}}
\newcommand{\tr}{\operatorname{tr}}
\newcommand{\sg}{\operatorname{sg}}
\newcommand{\cF}{\mathcal{F}}

\title{One-Step Next-Latent Prediction\\ Is Not a World Model}

\author{Shitong Wang \And Zhongang Cai \And Yuzhou Hong}

\begin{document}
\maketitle
\lhead{Preprint}

\begin{abstract}
Next-latent prediction fits a map from the current embedding to the next one.
LeNEPA carries this objective to time series, replacing the stop-gradient of next-embedding prediction with the isotropy penalty of LeJEPA.
A world model is a transition kernel that can be rolled out.
The one-step regression identifies a conditional mean, and a mean is a kernel only in special cases.
For a linear-Gaussian Markov latent, the mean transition and the innovation covariance are fixed by the one-step problem, and the open-loop squared error at horizon $K$ equals the trace of the sum of the pushed-forward innovation covariances.
That error grows with $K$ after the one-step fit is exact.
If the conditional mean is nonlinear, composing it is not the multi-step conditional mean.
If the observation is a non-injective function of a Markov state, a memoryless one-step map does not determine future observations, while a short window can.
An isotropy penalty is a function of the embedding marginal, so its partial derivative in the transition weights is zero.
On a scalar autoregression with coefficient $0.9$, the one-step mean squared error is $0.998$ and the $16$-step open-loop error is $5.10$.
On a hidden rotation, an eight-step window reaches $16$-step error $0.056$, while the current scalar alone reaches $0.778$.
Raising the isotropy weight from $0.1$ to $10$ leaves eight-step latent error inside $[0.78,0.85]$ on three seeds.
\end{abstract}

\section{Introduction}

A next-latent loss watches one step.
The input is the true latent, the target is the next latent, and the fitted object is a point.
A world model is a distribution over the next state, iterated on its own outputs.
\citet{ha20180122} roll that distribution inside a controller.
\citet{hafner20201603} train a policy on the imagined trajectory, and \citet{hafner20244104} keep the rollout as the model across domains.
\citet{xu2025nextembeddingpredictionmakesstrong} train the point-valued map for images.
\citet{chemeris2026lenepa} carry it to time series, using the isotropy penalty of \citet{balestriero20258544} where next-embedding prediction used a stop-gradient.
Figure~\ref{fig:flow} puts the two objects on one diagram.
The top row is what the training loss sees.
The bottom row is what imagination has to sample.
In the linear-Gaussian case they share a conditional mean and an innovation covariance, and they use those parameters differently.

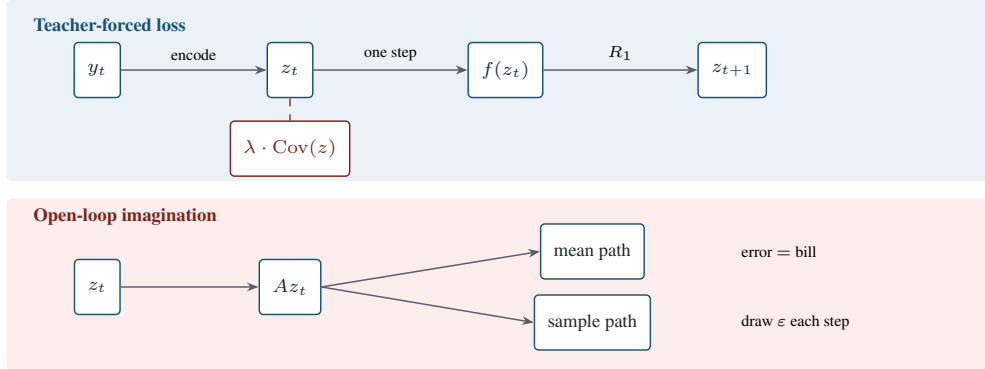
\begin{figure}[H]
\centering
\begin{tikzpicture}[
  font=\scriptsize,
  box/.style={draw=accent, rounded corners=1.6pt, fill=white, line width=0.55pt,
              minimum height=4.6ex, inner xsep=5pt, inner ysep=2pt, align=center, text=ink},
  arr/.style={-{Stealth[length=1.55mm]}, line width=0.55pt, draw=ink!75}
]
  \fill[lane, rounded corners=2pt] (-0.05,0.15) rectangle (13.05,2.55);
  \fill[lane2, rounded corners=2pt] (-0.05,-2.35) rectangle (13.05,-0.08);
  \node[anchor=west, font=\scriptsize\bfseries, text=accent] at (0.18,2.22) {Teacher-forced loss};
  \node[box] (y) at (1.15,1.62) {$y_t$};
  \node[box] (z) at (3.7,1.62) {$z_t$};
  \node[box] (f) at (6.55,1.62) {$f(z_t)$};
  \node[box] (tgt) at (9.55,1.62) {$z_{t+1}$};
  \node[box, draw=warn, text=warn] (pen) at (3.7,0.58) {$\lambda\cdot\mathrm{Cov}(z)$};
  \draw[arr] (y) -- node[above, font=\tiny] {encode} (z);
  \draw[arr] (z) -- node[above, font=\tiny] {one step} (f);
  \draw[arr] (f) -- node[above, font=\tiny] {$R_1$} (tgt);
  \draw[dashed, draw=warn, line width=0.5pt] (z.south) -- (pen.north);

  \node[anchor=west, font=\scriptsize\bfseries, text=warn] at (0.18,-0.32) {Open-loop imagination};
  \node[box] (z0) at (1.15,-1.25) {$z_t$};
  \node[box] (mean) at (3.7,-1.25) {$Az_t$};
  \node[box] (bill) at (7.7,-0.78) {mean path};
  \node[box] (smp) at (7.7,-1.72) {sample path};
  \node[font=\tiny, anchor=west] at (9.55,-0.78) {error $=$ bill};
  \node[font=\tiny, anchor=west] at (9.55,-1.72) {draw $\varepsilon$ each step};
  \draw[arr] (z0) -- (mean);
  \draw[arr] (mean.east) -- (bill.west);
  \draw[arr] (mean.east) -- (smp.west);
\end{tikzpicture}
\caption{Two uses of the same one-step mean.
The loss on the top row matches a point and attaches isotropy only to the marginal of $z_t$.
The bottom row forks after $Az_t$: the mean path pays the residual bill, and the sample path puts a fresh innovation back at every step.}
\label{fig:flow}
\end{figure}

Let $z_t\in\mathbb{R}^d$ be the latent and $f$ a measurable predictor, iterated by $f^{(1)}=f$ and $f^{(K)}=f\circ f^{(K-1)}$.
The two risks in Figure~\ref{fig:flow} are
\begin{equation}
R_1(f)=\Ex\|f(z_t)-z_{t+1}\|^2,
\qquad
R_K(f)=\Ex\|f^{(K)}(z_t)-z_{t+K}\|^2.
\label{eq:risks}
\end{equation}
Teacher forcing feeds $f$ the true latent at every step, so under stationarity its error equals $R_1(f)$ at every horizon.
\citet{chemeris2026lenepa} optimize that one-step regression plus a penalty on the marginal of the embedding.
The loss never writes down $R_K$.

Whether $R_1$ determines the rollout depends on which branch of Figure~\ref{fig:modes} the process sits in.
On the linear-Gaussian branch, $z_{t+1}=Az_t+\varepsilon_{t+1}$ with $\varepsilon_{t+1}$ independent of the past, mean zero, and covariance $\Sigma$.
Then $f(z)=Az$ attains $R_1(f)=\tr(\Sigma)$, and the open-loop error is the residual bill
\begin{equation}
R_K(f)=\tr\sum_{j=0}^{K-1} A^j\Sigma (A^j)^\top.
\label{eq:bill}
\end{equation}
The bill grows with $K$ after the one-step fit is exact.
The pair $(A,\Sigma)$ is the conditional law: $z\leftarrow Az+\varepsilon$ samples it, and $z\leftarrow Az$ returns only its mean.
On the middle branch the conditional mean bends, and composing it is not the two-step conditional mean.
On the right branch the training signal is a coordinate of a Markov state.
A single frame leaves a fiber of latent explanations; a short window can cut the fiber down to a point.
The isotropy penalty sits on the top row of Figure~\ref{fig:flow}.
Its partial derivative in the transition is zero, so it cannot choose among these branches.

Section~\ref{sec:theory} proves the three branches.
Section~\ref{sec:exp} measures them on an autoregression, a hidden rotation, and a learned encoder, where the conditional mean and the observation rank are known by construction.
The contributions are the bill \eqref{eq:bill}, the closure counterexample together with the window rank condition, and the separation between the isotropy weight and the transition.

\begin{figure}[H]
\centering
\begin{tikzpicture}[
  font=\scriptsize,
  panel/.style={draw=ink!30, rounded corners=2pt, line width=0.45pt},
  arr/.style={-{Stealth[length=1.35mm]}, line width=0.5pt, draw=ink!80}
]
  \begin{scope}
    \node[panel, minimum width=4.25cm, minimum height=3.35cm] at (2.1,1.65) {};
    \node[anchor=north, font=\scriptsize\bfseries, text=accent] at (2.1,3.15) {(a) Linear Markov};
    \draw[fill=lane, draw=accent] (0.7,1.85) ellipse (0.14 and 0.22);
    \draw[fill=lane, draw=accent] (1.95,1.85) ellipse (0.26 and 0.40);
    \draw[fill=lane, draw=accent] (3.4,1.85) ellipse (0.40 and 0.58);
    \draw[arr] (0.92,1.85) -- (1.58,1.85);
    \draw[arr] (2.30,1.85) -- (2.90,1.85);
    \node[font=\tiny] at (0.7,1.15) {$K{=}1$};
    \node[font=\tiny] at (1.95,1.15) {$K{=}2$};
    \node[font=\tiny] at (3.4,1.05) {$K$};
    \node[font=\tiny] at (2.1,0.45) {bill grows with $K$};
  \end{scope}
  \begin{scope}[shift={(4.45,0)}]
    \node[panel, minimum width=4.25cm, minimum height=3.35cm] at (2.1,1.65) {};
    \node[anchor=north, font=\scriptsize\bfseries, text=accent] at (2.1,3.15) {(b) Hidden coordinate};
    \draw[accent] (1.25,1.9) circle (0.55);
    \fill[warn] (1.55,2.22) circle (0.04);
    \fill[ink!55] (1.55,1.58) circle (0.04);
    \draw[dashed, warn] (1.55,1.40) -- (1.55,2.55);
    \node[font=\tiny, text=warn] at (1.25,1.12) {same $y_t$};
    \draw[arr] (2.0,1.9) -- (2.45,1.9);
    \node[draw=good, rounded corners=1pt, fill=good!8, font=\tiny, align=center, inner sep=2pt] at (3.25,1.9) {$y_{t-1}$\\$y_t$};
    \node[font=\tiny] at (2.1,0.45) {a window pins down $z$};
  \end{scope}
  \begin{scope}[shift={(8.9,0)}]
    \node[panel, minimum width=4.15cm, minimum height=3.35cm] at (2.05,1.65) {};
    \node[anchor=north, font=\scriptsize\bfseries, text=accent] at (2.05,3.15) {(c) Isotropy};
    \draw[fill=lane, draw=accent, rotate around={28:(0.85,2.05)}] (0.85,2.05) ellipse (0.42 and 0.22);
    \node[font=\tiny, anchor=north] at (0.85,1.62) {$\mathrm{Cov}$};
    \draw[arr] (1.5,2.05) -- node[above, font=\tiny] {$\lambda$} (2.15,2.05);
    \draw[fill=lane, draw=accent] (2.85,2.05) circle (0.28);
    \node[font=\tiny, anchor=north] at (2.85,1.68) {$I$};
    \node[draw=warn, rounded corners=1pt, fill=white, font=\tiny, inner sep=2pt] at (2.05,0.7) {$\partial L/\partial A=0$};
  \end{scope}
\end{tikzpicture}
\caption{Three ways a one-step fit comes apart from a rollout.
(a)~Once $A$ and $\sigma^2$ are known, each extra open-loop step adds a pushed-forward innovation.
(b)~Two latent states share the dashed coordinate $y_t$ and disagree about the future; the window $(y_{t-1},y_t)$ separates them.
(c)~$\lambda$ rounds the marginal of $z$ toward $I$ and does not enter the transition.}
\label{fig:modes}
\end{figure}
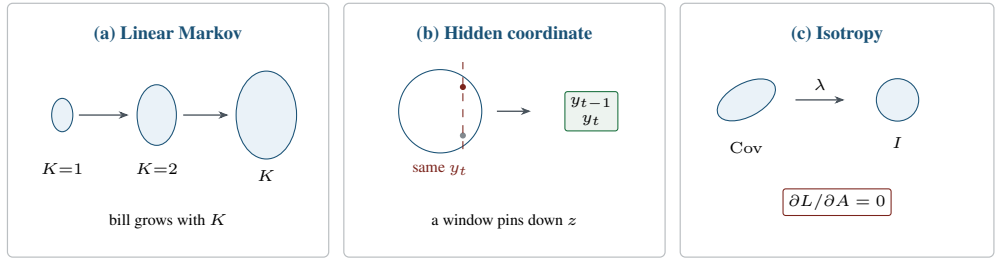

\section{Related work}

\paragraph{Next-embedding prediction.}
\citet{xu2025nextembeddingpredictionmakesstrong} train a vision transformer to match the next patch embedding, and stop the gradient so the target cannot collapse to a constant.
The stop-gradient is the device of \citet{grill2020byol} and of \citet{chen2021simsiam}.
The match itself is a point in representation space, not a draw from a transition.
\citet{assran2023jepa} make that split explicit for images.
\citet{bardes2024vjepa} extend it to video, and \citet{bardes20224906} keep the embedding alive with a covariance penalty rather than a moving-average target.
\citet{balestriero20258544} replace those heuristics by SIGReg, an isotropy constraint aimed at the known collapse modes of joint-embedding training.
\citet{chemeris2026lenepa} import the constraint into time series: a causal transformer \citep{vaswani2017transformer}, mean squared error onto the next latent token, and no augmentation.
The same point-valued match is what masked autoencoders regress after a heavy mask \citep{he2022mae}, and what contrastive and clustering objectives shape without an explicit next token \citep{chen2020simclr,caron2021dino}.
Their reported evaluation is the representation.
The evaluation here is whether the fitted map is a rollout kernel.

\paragraph{World models.}
\citet{ha20180122} learn a latent transition and roll it out inside the controller.
\citet{hafner20194551} plan in a latent dynamics model trained from pixels.
\citet{hafner20201603} go further and train the policy on trajectories drawn from the transition, so the learned object has to be a kernel.
\citet{hafner20222193} and \citet{hafner20244104} keep that imagined trajectory as the training domain of the agent.
IRIS rolls out a transformer world model \citep{micheli20230588}, TD-MPC2 plans with one \citep{hansen20246828}, and \citet{wu20224176} run the same idea on a physical robot.
\citet{schrittwieser20208265} plan with a learned model whose outputs are values, policies, and rewards, which is again a multi-step object.
Later pixel world models ask that kernel to render: diffusion world models \citep{alonso20242399}, interactive environments \citep{bruce20245391}, and driving \citep{hu20237080}.
The newest open simulators push that rollout into interactive video.
LingBot-World holds a minute-scale horizon and distills a causal sampler so a control step returns within a second \citep{robbyant2026lingbotworld}.
Its successor trains the causal generator for an unbounded horizon and distills a real-time 720p model \citep{gao2026lingbotinfinity}.
\citet{zhang2026echowm} answer a continuous 6-DoF camera trajectory with video, environmental sound, and speech together, and the long horizon is trained autoregressively on the model's own frames.
That post-training is the open-loop gap in Figure~\ref{fig:flow} showing up in a video generator: a denoiser scored on ground-truth context drifts once the context is the sample path.
The theorems below say which training signals identify that object.
The experiments do not reimplement the agents.

\paragraph{Forecasting.}
\citet{zhou20217436} and \citet{nie20234730} train and score transformers on long forecast windows.
\citet{wu20232186} represent a series by a two-dimensional temporal variation.
\citet{zeng20223504} find that a linear map is often a strong baseline, which is what affine closure predicts: once the state is Markov and the mean is linear, extra depth does not change the conditional mean.
\citet{oord2018cpc} also score future latents, through a density ratio rather than a squared residual.
Decomposition transformers and time-series foundation models are scored on the same long window \citep{wu20223008,ansari20247815,yue20220466}.
On the generative side, \citet{ho2020ddpm} learn a one-step denoiser and \citet{song2021ddim} show that the sample the user sees is a multi-step path built from that denoiser.
The residual bill is the linear-Gaussian version of the same split.
An exact one-step map still has to inject noise at every imagined step if the draw is supposed to follow the conditional law.
Appendix~\ref{app:related} places this split among pretext tasks, autoregressive generators, diffusion samplers, and forecasting models.

\section{What one-step regression identifies}
\label{sec:theory}

\subsection{Setup}

Let $(\Omega,\cF,\mathbb{P})$ carry the latent process $(z_t)_{t\in\mathbb{Z}}$ in $\mathbb{R}^d$, with $\Ex\|z_t\|^2<\infty$.
Write $\cF_t=\sigma(z_s:s\le t)$.
The process is Markov if $\Ex[z_{t+1}\mid \cF_t]=\Ex[z_{t+1}\mid z_t]$.
The one-step regression problem is $\inf_f R_1(f)$ over measurable $f$.
Its solution is the conditional mean, by the usual $L^2$ projection: for every square-integrable $f$,
\begin{equation}
R_1(f)=R_1(f_\star)+\Ex\|f(z_t)-f_\star(z_t)\|^2,
\qquad
f_\star(z)=\Ex[z_{t+1}\mid z_t=z],
\label{eq:projection}
\end{equation}
whenever the process is Markov and stationary, so that one time index describes every step.
The minimal value equals $\Ex[\tr\Cov(z_{t+1}\mid z_t)]$.
Any noise law with the same conditional mean leaves $R_1$ unchanged.
Squared error on the next latent therefore identifies $f_\star$ and the trace of the residual covariance.
It does not identify the residual law.

Teacher forcing makes this easy to miss in a training curve.
\begin{proposition}[Teacher forcing hides the bill]
\label{prop:tf}
Suppose $(z_t)$ is strictly stationary and $f_\star(z)=\Ex[z_{t+1}\mid z_t=z]$.
Then for every $K\ge 1$,
\[
\Ex\|f_\star(z_{t+K-1})-z_{t+K}\|^2=R_1(f_\star).
\]
The open-loop risk $R_K(f_\star)$ equals this number for every $K$ only in degenerate cases, characterized below.
\end{proposition}

\begin{proof}
Stationarity equates the joint law of $(z_{t+K-1},z_{t+K})$ with the joint law of $(z_t,z_{t+1})$.
The identity $R_K=R_1$ for all $K$ fails as soon as the sum in \eqref{eq:bill} has a nonzero term beyond $j=0$, which is Theorem~\ref{thm:bill}.
\end{proof}

\subsection{Linear-Gaussian Markov latents}

\begin{theorem}[Residual bill]
\label{thm:bill}
Let $z_{t+1}=Az_t+\varepsilon_{t+1}$ with $\Ex[\varepsilon_{t+1}\mid \cF_t]=0$, $\Cov(\varepsilon_{t+1}\mid \cF_t)=\Sigma$, and $\varepsilon_{t+1}$ independent of $\cF_t$.
Set $f(z)=Az$.
Then $f=f_\star$, $R_1(f)=\tr(\Sigma)$, $\Ex[z_{t+K}\mid \cF_t]=A^K z_t$, and
\begin{equation}
\Cov(z_{t+K}\mid \cF_t)=\sum_{j=0}^{K-1} A^j\Sigma (A^j)^\top,
\qquad
R_K(f)=\tr\sum_{j=0}^{K-1} A^j\Sigma (A^j)^\top.
\label{eq:covk}
\end{equation}
If $d=1$, $A=(a)$, and $\Sigma=(\sigma^2)$, this reduces to
\begin{equation}
R_K(f)=\sigma^2\frac{1-a^{2K}}{1-a^2}\qquad(a^2\neq 1),
\label{eq:scalar}
\end{equation}
and to $R_K(f)=K\sigma^2$ if $a^2=1$.
In every dimension,
\begin{equation}
R_{K+1}(f)=R_K(f)+\tr\!\big(A^K\Sigma(A^K)^\top\big)\ge R_K(f),
\label{eq:increment}
\end{equation}
and the increment is strictly positive whenever $\Sigma$ is positive definite and $A^K$ is invertible.
In one dimension the increment equals $\sigma^2 a^{2K}$, so the risk is strictly increasing in $K$ whenever $\sigma^2>0$ and $a\neq 0$.
\end{theorem}

\begin{proof}
The conditional mean of $\varepsilon_{t+1}$ is zero, so $f_\star(z)=Az$.
Unrolling the recurrence gives
\begin{equation}
z_{t+K}=A^K z_t+\sum_{j=0}^{K-1} A^j\varepsilon_{t+K-j}.
\label{eq:unroll}
\end{equation}
Each innovation is independent of $\cF_t$ and of the other innovations, with mean zero, so the conditional expectation of the sum is zero and the cross terms in the conditional covariance vanish.
Taking the trace and the expectation removes the conditioning because the conditional covariance does not depend on $z_t$.
In one dimension the sum is geometric.
Horizon $K+1$ appends one summand, $A^K\Sigma(A^K)^\top$, which is the increment \eqref{eq:increment}.
Its trace is nonnegative, and it is positive when $\Sigma$ is positive definite and $A^K$ has full rank.
\end{proof}

The last sentence of the proof is the fact a training curve cannot see.
After $A$ has been learned perfectly, every extra step of open-loop rollout adds a fresh copy of $\Sigma$, rotated by the powers of $A$.
The one-step loss has already reached its floor $\tr(\Sigma)$ and stays there under teacher forcing.
The rollout error keeps the bill.

The same unrolling says what a world model must sample.
\begin{corollary}[Mean iteration and kernel iteration]
\label{cor:sample}
Under the hypotheses of Theorem~\ref{thm:bill}, the conditional law of $z_{t+K}$ given $\cF_t$ is determined by $(A,\Sigma)$ whenever $\varepsilon_t$ is Gaussian.
The deterministic rollout $\hat z_{t+1}=A\hat z_t$ returns $\Ex[z_{t+K}\mid z_t]$ and has risk \eqref{eq:covk}.
The recursion $\tilde z_{t+1}=A\tilde z_t+\varepsilon_{t+1}$ with fresh innovations returns a sample of that conditional law.
Both recursions are fixed by the one-step normal equation for $A$ and by the residual covariance $\Sigma=\Cov(z_{t+1}-Az_t)$.
\end{corollary}

So, inside this class, one-step regression is enough to build the world model, and it is not itself the world model.
The map $f(z)=Az$ that minimizes $R_1$ emits a point.
The kernel emits a cloud of radius given by the bill.
Reporting $R_1$ as evidence of multi-step generation reports the radius of one innovation and omits the sum.

For the scalar process used in the experiment, $a=0.9$ and $\sigma^2=1$, equation \eqref{eq:scalar} gives the ratio
\[
\frac{R_{16}}{R_1}=\frac{1-0.9^{32}}{1-0.9^2}=\frac{1-0.9^{32}}{0.19}\approx 5.082.
\]
The measured innovation variance on the finite sample is $0.9982$, so the predicted sixteen-step risk is $5.073$.
Section~\ref{sec:exp} records an empirical open-loop error of $5.100$ for the fitted one-step coefficient $0.9001$.

\subsection{Nonlinear conditional means}

Affine maps commute with conditional expectation.
Other maps do not.
\begin{theorem}[Affine closure]
\label{thm:affine}
Let $(z_t)$ be Markov and let $f(z)=\Ex[z_{t+1}\mid z_t=z]$.
If $f(z)=Bz+c$, then $\Ex[z_{t+K}\mid z_t]=f^{(K)}(z_t)$ for every $K\ge 1$.
Conversely, there exists a Markov process whose one-step conditional mean is nonlinear and for which the identity already fails at $K=2$: if $z_{t+1}=z_t^2+\varepsilon_{t+1}$ with $\varepsilon_{t+1}\sim\mathcal{N}(0,1)$ independent of $\cF_t$, then $f(u)=u^2$, $f(f(0))=0$, and $\Ex[z_{t+2}\mid z_t=0]=1$.
\end{theorem}

\begin{proof}
If $f(z)=Bz+c$, the tower property gives
\begin{align*}
\Ex[z_{t+2}\mid z_t]
&=\Ex[\Ex[z_{t+2}\mid z_{t+1}]\mid z_t]
=\Ex[Bz_{t+1}+c\mid z_t]\\
&=B(Bz_t+c)+c=f(f(z_t)).
\end{align*}
The inductive step from $K$ to $K+1$ is the same, because $f$ pulls out of the conditional expectation.
For the counterexample, $f(u)=u^2$ and $f(f(0))=0$.
Given $z_t=0$, one has $z_{t+1}=\varepsilon_{t+1}$ and $z_{t+2}=\varepsilon_{t+1}^2+\varepsilon_{t+2}$, whose expectation is $\Ex[\varepsilon_{t+1}^2]=1$.
The two numbers differ by the conditional variance of $z_{t+1}$, which is the Jensen gap of $u\mapsto u^2$.
\end{proof}

The counterexample does not need a stationary process.
It needs a state, here $0$, at which the conditional law of the next latent is dispersed and $f$ bends.
A one-step training loss still drives $f$ toward $u\mapsto u^2$, because that is the conditional mean.
Rolling that $f$ out deterministically from $0$ predicts $0$ two steps ahead.
The true conditional mean is $1$.
Multi-step regression on the iterated map would see this gap.
One-step regression cannot, because both maps agree at horizon one: $\Ex[z_{t+1}\mid z_t=0]=0$.

\subsection{Partial observations}

World-model training often sees an observation $y_t$, not the latent.
LeNEPA's tokens are functions of a window of the series, so the practical question is whether that window is a state.
\begin{theorem}[Memoryless maps on a rotation]
\label{thm:obs}
Let $z_t\in\mathbb{R}^2$ and $z_t=R_\theta z_{t-1}$, where $R_\theta$ is rotation by $\theta$ with $\sin\theta\neq 0$.
Let $y_t=e_1^\top z_t$.
No measurable map $g:\mathbb{R}\to\mathbb{R}$ satisfies $g(y_t)=y_{t+1}$ for all $z_t$.
The linear map
\begin{equation}
\Phi z=
\begin{pmatrix}
1&0\\
\cos\theta&-\sin\theta
\end{pmatrix}
z
\label{eq:phi}
\end{equation}
sends $z_{t-1}$ to $(y_{t-1},y_t)$, has determinant $-\sin\theta\neq 0$, and therefore $z_t$ is a function of $(y_{t-1},y_t)$.
Consequently $y_{t+K}$ is a function of that window for every $K$.
\end{theorem}

\begin{proof}
From the rotation, $y_{t+1}=\cos\theta\, y_t-\sin\theta\, z_{t,2}$.
If $z_{t,2}$ can be either of two distinct values at fixed $y_t$, the two futures differ by a nonzero multiple of $\sin\theta$.
A map of $y_t$ alone takes one value and cannot match both.
The matrix in \eqref{eq:phi} is the change of variables $(z_{t-1,1},z_{t-1,2})\mapsto(y_{t-1},y_t)$, and its determinant is $-\sin\theta$.
Invertibility gives $z_{t-1}$, then $z_t=R_\theta z_{t-1}$, then every future observation.
\end{proof}

With observation noise the equality becomes an estimation statement.
The experiment in Section~\ref{sec:exp} uses process noise $0.05$ and observation noise $0.05$.
The rank condition explains the shape of the result: a window of length eight, which is longer than the two coordinates the theorem requires, drives sixteen-step error from $0.778$ down to $0.056$, and training that window with a one-step loss is enough.
The multi-step unroll is not the ingredient that repairs a non-Markov observation.
The state is.

\subsection{Isotropy does not see the transition}

Let $y_t\in\mathbb{R}^{d_y}$ and let an encoder and a transition be bias-free linear maps $z=Wy$ and $\hat z=Az$, with $W\in\mathbb{R}^{d\times d_y}$ and $A\in\mathbb{R}^{d\times d}$.
The training objective used in the third experiment, and the one that isolates the LeJEPA penalty from the transition, is
\begin{equation}
L(W,A)
=\Ex\|AWy_t-\sg(Wy_{t+1})\|^2
+\lambda\,\|\widehat\Cov(Wy)-I\|_F^2.
\label{eq:liso}
\end{equation}
The stop-gradient $\sg$ is the NEPA device.
LeNEPA's published training replaces it by the isotropy penalty; the decomposition below is the reason that replacement still leaves $A$ unsupervised by $\lambda$.
The covariance is the empirical covariance of the batch of embeddings, or its population counterpart.
In either case it is a function of the marginal pushforward of $y$ through $W$.

\begin{theorem}[The penalty has no transition gradient]
\label{thm:iso}
The second term of \eqref{eq:liso} does not depend on $A$.
Its gradient with respect to $A$ is zero.
At a stationary encoder $W$, the minimizing transition is the one-step coefficient
\begin{equation}
A_\star(W)=\Ex[z_{t+1}z_t^\top]\,\Ex[z_t z_t^\top]^{+},
\qquad z_t=Wy_t,
\label{eq:astar}
\end{equation}
on the range of $\Ex[z_t z_t^\top]$, and $\lambda$ enters $A_\star$ only through $W$.
If $W$ and $W'$ induce the same joint second-moment matrix of $(z_t,z_{t+1})$, they induce the same predictive term and the same $A_\star$, whether or not they induce the same marginal isotropy.
\end{theorem}

\begin{proof}
The stop-gradient blocks $W$ inside the target, so for fixed $W$ the predictive term is a linear least-squares problem in $A$ whose solution is \eqref{eq:astar}.
The Frobenius penalty is constant with respect to $A$.
The last sentence is the normal equation: $A_\star$ and the minimal predictive loss are functions of $\Ex[z_t z_t^\top]$ and $\Ex[z_{t+1}z_t^\top]$ only.
\end{proof}

Isotropy can still change the rollout, because a joint minimization moves $W$, and $A_\star$ depends on $W$.
What it cannot do is prefer, at fixed $W$, the transition that rolls out farther.
Once $W$ is isotropic enough that the penalty is near its floor, further increases of $\lambda$ have nothing to change in the marginal and no channel into $A$ except numerical noise.
The third experiment is this comparative static: $\lambda\in\{0,0.1,1,10\}$.

\section{Experiments}
\label{sec:exp}

All runs use seed $0$ for the curves in Figure~\ref{fig:panels}.
Seeds $1$ and $2$ are reported where an optimization is initialization-sensitive.
Data are synthesized on one GPU.
The autoregression uses $4096$ sequences of length $64$.
The rotation uses $4096$ sequences of length $80$.
The encoder uses $2048$ sequences of length $48$.
Optimization is Adam.
Exact widths, step counts, and learning rates are in the appendix.
No number in the tables was filled in before the corresponding run finished.

\begin{figure}[H]
\centering
\includegraphics[width=\textwidth]{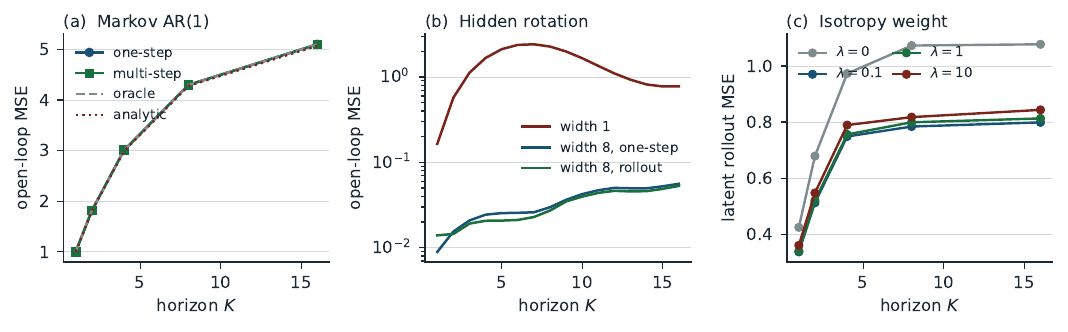}
\caption{Seed $0$.
(a) Scalar AR(1), $a=0.9$.
One-step regression, an eight-step unroll, and the oracle coefficient sit on the analytic residual bill.
(b) A planar rotation observed through one coordinate.
The memoryless curve follows the autocorrelation of the scalar observation, including the near-period at $K\approx 2\pi/0.4$.
A window of length eight keeps the open-loop error small under either a one-step or a rollout loss.
(c) Latent rollout error of a learned linear encoder and transition at four isotropy weights.
The flatness for $\lambda\ge 0.1$ is the quantity checked on two further seeds in Table~\ref{tab:iso}.}
\label{fig:panels}
\end{figure}

\subsection{Autoregression, where one step is enough for the mean}

The process is $x_{t+1}=0.9\,x_t+\varepsilon_{t+1}$ with $\varepsilon_{t+1}\sim\mathcal{N}(0,1)$ and $x_0=0$.
A scalar coefficient $a$ is fit by Adam for $400$ steps.
The one-step objective uses every consecutive pair.
The multi-step objective unrolls eight steps from a random time index and averages the open-loop squared error.
The fitted values are $a=0.9001$ and $a=0.8930$.
Open-loop error is then computed in closed form on the same sample, as $\Ex[(x_{t+K}-a^K x_t)^2]$, against the oracle $a=0.9$ and against \eqref{eq:scalar} with the sample innovation variance $0.9982$.

\begin{table}[H]
\caption{Open-loop MSE on the AR(1).
The analytic column is \eqref{eq:scalar} at $a=0.9$ with $\sigma^2=0.9982$.
Teacher-forced one-step error equals the $K=1$ entry at every horizon.}
\label{tab:ar}
\centering
\small
\begin{tabular}{lrrrr}
\toprule
$K$ & one-step $a=0.9001$ & unroll $a=0.8930$ & oracle $0.9$ & analytic \\
\midrule
1 & 0.998 & 0.998 & 0.998 & 0.998 \\
2 & 1.813 & 1.813 & 1.813 & 1.807 \\
4 & 3.007 & 3.009 & 3.007 & 2.992 \\
8 & 4.298 & 4.300 & 4.298 & 4.280 \\
16 & 5.100 & 5.101 & 5.100 & 5.073 \\
\bottomrule
\end{tabular}
\end{table}

Table~\ref{tab:ar} is Theorem~\ref{thm:bill} on a finite sample.
The one-step coefficient matches the oracle to $10^{-4}$, and the sixteen-step error is still five times the one-step error.
The eight-step training objective does not undercut the bill: its coefficient is slightly short of $0.9$ after $400$ steps, and its rollout curve lies on top of the one-step curve.
The gap between the empirical $5.100$ and the analytic $5.073$ is the finite-sample gap in the second moment, not a missing loss term.
A planner that rolls the conditional mean forward for sixteen steps pays about $5.10$, not $0.998$.
A planner that samples $\varepsilon$ at each step draws from the conditional law whose variance is that same number.
Both behaviors are fixed once $a$ and $\sigma^2$ are known, and both are known from one step.

\subsection{Hidden rotation, where the state is the missing object}

The latent is planar, $z_{t+1}=R_\theta z_t+\eta_{t+1}$ with $\theta=0.4$ and $\eta_{t+1}\sim\mathcal{N}(0,0.05^2 I)$.
The observation is the first coordinate plus $\mathcal{N}(0,0.05^2)$ noise.
A memoryless predictor is a network with two hidden layers of width $64$ reading $y_t$, trained with one-step squared error.
A window predictor reads eight consecutive observations.
It is trained either with one-step error or with an eight-step open-loop unroll.
Evaluation is open loop from time $20$ on $2000$ sequences, feeding predictions back into the window.

\begin{table}[H]
\caption{Open-loop MSE on the hidden rotation, seed $0$.
Seeds $1$ and $2$ keep the same ordering: memoryless error at $K=8$ is $2.602$ and $2.537$, and windowed one-step error at $K=8$ is $0.031$ and $0.030$.}
\label{tab:rot}
\centering
\small
\begin{tabular}{lrrr}
\toprule
$K$ & width $1$, one-step & width $8$, one-step & width $8$, unroll \\
\midrule
1 & 0.165 & 0.009 & 0.014 \\
2 & 0.572 & 0.015 & 0.015 \\
4 & 1.670 & 0.024 & 0.021 \\
8 & 2.268 & 0.030 & 0.027 \\
16 & 0.778 & 0.056 & 0.053 \\
\bottomrule
\end{tabular}
\end{table}

Table~\ref{tab:rot} matches Theorem~\ref{thm:obs} once noise is added.
The memoryless error grows past $2$ and then falls to $0.778$ at $K=16$.
The fall is the period of the rotation: $2\pi/0.4\approx 15.7$, so a sixteen-step forecast lands near the same phase as the present, and a predictor that ignores the hidden coordinate is scored against a returning autocorrelation.
It has not recovered the state.
The window of length eight, one longer than the two-dimensional latent and long enough to average the observation noise, reaches $0.056$ at the same horizon with a one-step loss.
The unroll loss is slightly worse at $K=1$ ($0.014$ against $0.009$) and slightly better at $K=16$ ($0.053$ against $0.056$).
That margin is consistent with Theorem~\ref{thm:affine}: after the window restores an approximate Markov state, the conditional mean is nearly linear, one-step regression identifies it, and an unroll does not buy a different mean.
The failure mode of the memoryless column is the sigma-algebra, not the horizon inside the loss.

\subsection{Isotropy, where the penalty fixes the marginal}

The latent is two-dimensional with
\[
A_{\mathrm{true}}=\begin{pmatrix}0.8&-0.3\\0.2&0.7\end{pmatrix},
\]
innovation variance $0.3^2$, a random $4\times 2$ emission, and observation noise $0.1$.
The model is a bias-free map $W\in\mathbb{R}^{2\times 4}$ and a bias-free $A\in\mathbb{R}^{2\times 2}$, trained for $500$ Adam steps at learning rate $10^{-2}$ on \eqref{eq:liso}.
The target embedding is stop-gradient.
Open-loop error is measured in the encoder's coordinates, from time $10$, on $512$ sequences: apply $A$ for $K$ steps to $Wy_t$ and compare with $Wy_{t+K}$.
This is the rollout of the learned latent, which is the quantity Theorem~\ref{thm:iso} constrains.

\begin{table}[H]
\caption{Eight-step latent rollout MSE and off-diagonal energy $\|(C-\operatorname{diag}(C))\|_F^2/4$, three seeds.
For $\lambda\ge 0.1$ every eight-step entry lies in $[0.781,0.847]$.}
\label{tab:iso}
\centering
\small
\begin{tabular}{lccc ccc}
\toprule
& \multicolumn{3}{c}{eight-step MSE} & \multicolumn{3}{c}{off-diagonal energy} \\
\cmidrule(lr){2-4}\cmidrule(lr){5-7}
$\lambda$ & seed 0 & seed 1 & seed 2 & seed 0 & seed 1 & seed 2 \\
\midrule
$0$ & 1.073 & 2.047 & 0.419 & 0.132 & 0.088 & 0.019 \\
$0.1$ & 0.784 & 0.847 & 0.781 & $8\cdot 10^{-6}$ & $8\cdot 10^{-4}$ & 0.004 \\
$1$ & 0.799 & 0.784 & 0.809 & 0.001 & $3\cdot 10^{-4}$ & $8\cdot 10^{-5}$ \\
$10$ & 0.818 & 0.845 & 0.827 & 0.002 & $2\cdot 10^{-8}$ & $6\cdot 10^{-5}$ \\
\bottomrule
\end{tabular}
\end{table}

At $\lambda=0$ the eight-step error depends on the initialization: $1.073$, $2.047$, and $0.419$.
Seed $2$ without a penalty rolls out better than the same seed with a penalty ($0.419$ against $0.781$ at $\lambda=0.1$), while its off-diagonal energy falls from $0.019$ to $0.004$.
The penalty can improve isotropy and worsen the rollout at the same time, which is possible only because the two terms score different variables.
The nine entries with $\lambda\ge 0.1$ all lie in $[0.781,0.847]$.
The largest swing inside one seed on that range is seed $1$, from $0.784$ at $\lambda=1$ to $0.847$ at $\lambda=0.1$.
Theorem~\ref{thm:iso} predicts the flatness: after the covariance term is near zero, $\lambda$ has no remaining first-order effect on $A$.
The drop from the worst $\lambda=0$ run to the isotropic band is a change in $W$, which the theorem allows and does not interpret as transition supervision.

\section{Discussion}

The positive statement is narrow and useful.
Inside linear-Gaussian Markov models, a next-latent regression that has reached $R_1=\tr(\Sigma)$ has already estimated the world model.
The remaining implementation step is to store $\Sigma$ and to add a fresh Gaussian innovation at each imagined step.
Omitting that draw yields the conditional mean, whose squared error is \eqref{eq:bill} and whose sixteen-step value on the AR(1) is $5.10$ rather than the training loss $0.998$.
Publishing the training loss as a multi-step generation metric reports $R_1$ under teacher forcing, which Proposition~\ref{prop:tf} says is invariant to the horizon.

The negative statements are the ones that apply to a representation trained on raw observations.
If the token does not contain a state, Theorem~\ref{thm:obs} says that no head, one-step or multi-step, can invent the missing coordinate from a single frame.
LeNEPA's causal window is doing the work that the width-$8$ column does in Table~\ref{tab:rot}, and the isotropy penalty is doing the work that $\lambda$ does in Table~\ref{tab:iso}.
Neither ingredient is a substitute for checking open-loop error against the residual bill.
If the conditional mean bends, Theorem~\ref{thm:affine} says that a correct one-step head, iterated, is the wrong multi-step mean, and the gap is visible at horizon two.

The experiments are not an Atari or a robot benchmark.
They are processes for which the conditional mean, the innovation covariance, and the observation rank are known, so a disagreement would have falsified the algebra rather than the optimizer.
The encoder is linear.
The rotation network has two hidden layers of width $64$.
Three seeds are enough to see that $\lambda=0$ is initialization-sensitive and that $\lambda\in[0.1,10]$ is not, and they are not a confidence interval for a large transformer.
SIGReg's sliced estimator is not reimplemented.
The penalty in \eqref{eq:liso} is the squared distance of the covariance to the identity, the second-moment part of the isotropy constraint in \citet{balestriero20258544}.
Extending the bill to nonlinear $f$ with additive noise is immediate for the mean and false for the identity $f^{(K)}=\Ex[z_{t+K}\mid z_t]$, as the quadratic counterexample shows.
The stochastic completion in Corollary~\ref{cor:sample} needs a residual law, which squared error does not provide beyond the covariance.

\section{Conclusion}

One-step next-latent prediction fits a conditional mean and, together with the residual covariance, determines every horizon of a linear-Gaussian Markov model.
The deterministic rollout of that mean is not a sample from the conditional law, and its squared error is the residual bill, which on the AR(1) grows from $0.998$ to $5.10$ by horizon sixteen.
Outside that class the same loss does not determine the rollout: nonlinear means break closure under composition, partial observations break it under a memoryless map, and an isotropy penalty does not grade the transition.
A world model is the kernel obtained by putting the innovation back into the recursion.
The one-step map is the mean of that kernel.

\bibliography{refs}
\bibliographystyle{iclr2026_conference}

\appendix
\section{Proof details}
\label{app:proofs}

\subsection{Cross terms in the residual bill}

Start from \eqref{eq:unroll}.
Write $e_K=\sum_{j=0}^{K-1} A^j\varepsilon_{t+K-j}$.
For $j\neq \ell$, the tower property and the assumption $\Ex[\varepsilon_{s}\mid \cF_{s-1}]=0$ give
\[
\Ex[\varepsilon_{t+K-j}\varepsilon_{t+K-\ell}^\top\mid \cF_t]=0,
\]
because the later innovation has conditional mean zero given the earlier one.
Hence
\[
\Ex[e_K e_K^\top\mid \cF_t]
=\sum_{j=0}^{K-1} A^j\,\Ex[\varepsilon_{t+K-j}\varepsilon_{t+K-j}^\top\mid \cF_t]\,(A^j)^\top
=\sum_{j=0}^{K-1} A^j\Sigma(A^j)^\top.
\]
The squared norm is the trace.
Taking $\Ex$ and using independence of the conditional covariance from $z_t$ yields $R_K(f)$.
If $a^2=1$ in one dimension, each of the $K$ summands equals $\sigma^2$, so $R_K=K\sigma^2$.
If $a^2\neq 1$,
\[
\sum_{j=0}^{K-1}a^{2j}=\frac{1-a^{2K}}{1-a^2}.
\]

Let $S_K=\sum_{j=0}^{K-1} A^j\Sigma(A^j)^\top$.
Then $S_{K+1}=S_K+A^K\Sigma(A^K)^\top$, which is \eqref{eq:increment}.
The added matrix is positive semidefinite, and it is positive definite when $\Sigma$ is positive definite and $A^K$ is invertible.
Shifting the time index in \eqref{eq:unroll} also gives the decomposition $e_K=\varepsilon_{t+K}+A e'_{K-1}$ with $e'_{K-1}$ equal in law to $e_{K-1}$ and independent of $\varepsilon_{t+K}$, so
\[
R_K(f)=\Ex\|A e_{K-1}\|^2+\tr(\Sigma)\ge \tr(\Sigma)=R_1(f).
\]
In the scalar AR(1) the one-horizon increment is exactly $\sigma^2 a^{2(K-1)}$, which is positive for every $K$ when $\sigma^2>0$ and $a\neq 0$.

\subsection{Affine induction}

Assume $\Ex[z_{t+1}\mid z_t]=Bz_t+c$.
The claim $\Ex[z_{t+K}\mid z_t]=f^{(K)}(z_t)$ holds at $K=1$.
If it holds at $K$, then
\begin{align*}
\Ex[z_{t+K+1}\mid z_t]
&=\Ex[\Ex[z_{t+K+1}\mid z_{t+1}]\mid z_t]\\
&=\Ex[f^{(K)}(z_{t+1})\mid z_t].
\end{align*}
Affinity of $f$ implies affinity of $f^{(K)}$, so $f^{(K)}(z)=B^K z+(I+B+\cdots+B^{K-1})c$ and
\begin{align*}
\Ex[f^{(K)}(z_{t+1})\mid z_t]
&=B^K\Ex[z_{t+1}\mid z_t]+(I+\cdots+B^{K-1})c\\
&=B^K(Bz_t+c)+(I+\cdots+B^{K-1})c\\
&=f^{(K+1)}(z_t).
\end{align*}
The quadratic counterexample is a direct expansion and does not use this induction.
Given $z_t=0$, $z_{t+1}=\varepsilon_{t+1}$ and
\[
\Ex[z_{t+2}\mid z_t=0]=\Ex[\varepsilon_{t+1}^2+\varepsilon_{t+2}]=\Ex[\varepsilon_{t+1}^2]=1,
\]
while $f(f(0))=0^2=0$.
The gap equals $\Var(z_{t+1}\mid z_t=0)$ because $f(u)=u^2$ and $\Ex[\varepsilon^2]=(\Ex\varepsilon)^2+\Var(\varepsilon)$.

\subsection{Determinant of the window map}

Let $c=\cos\theta$ and $s=\sin\theta$, and
\[
R_\theta=\begin{pmatrix}c&-s\\ s&c\end{pmatrix}.
\]
Then $z_t=R_\theta z_{t-1}$ expands to $y_t=c\,z_{t-1,1}-s\,z_{t-1,2}$ and $y_{t-1}=z_{t-1,1}$.
In matrix form that is \eqref{eq:phi}.
The determinant equals $-s$.
The hypothesis $\sin\theta\neq 0$ is therefore necessary and sufficient for the two most recent noiseless observations to be a linear coordinate system on $\mathbb{R}^2$.
Applying $R_\theta$ gives $z_t$, and applying it $K$ times gives $y_{t+K}=e_1^\top R_\theta^{K} z_t$.

With additive observation noise the same matrix is the Jacobian of the conditional mean in a linear-Gaussian smoother.
The experiment does not invert $\Phi$ by hand.
It trains a network on the window, which is a nonparametric stand-in for that linear inverse, and the error drop in Table~\ref{tab:rot} is the empirical content of $\det\Phi\neq 0$.

\subsection{Normal equation for the transition}

Fix $W$ and write $z_t=Wy_t$, $u_t=\sg(Wy_{t+1})$.
The predictive term is $\Ex\|Az_t-u_t\|^2$.
Expanding in coordinates and differentiating in $A$ produces the matrix equation
\[
A\,\Ex[z_t z_t^\top]=\Ex[u_t z_t^\top],
\]
which is \eqref{eq:astar} on the range of the Gram matrix.
The penalty's derivative in any entry of $A$ is zero because the penalty does not contain $A$.
Differentiating through $W$ is where $\lambda$ acts.
That derivative can be nonzero, and Table~\ref{tab:iso} at $\lambda=0$ versus $\lambda=0.1$ is its effect.
Differentiating through $A$ at fixed $W$ is the channel the theorem shuts.

\section{Experimental protocol}
\label{app:protocol}

Seeds are NumPy and PyTorch seeds set together before each run.
The reported curves in Figure~\ref{fig:panels} use seed $0$.
Table~\ref{tab:iso} and the parenthetical checks in Section~\ref{sec:exp} use seeds $0,1,2$.
Hardware is one GPU.
The implementation is \texttt{experiments/run\_lenepa\_wm.py}.

\paragraph{AR(1).}
Sample size $4096$, length $64$, $a=0.9$, innovation standard deviation $1$, initial state $0$.
The coefficient is a single unconstrained scalar, optimized by Adam with learning rate $0.05$ for $400$ steps.
The one-step loss averages $(a x_t-x_{t+1})^2$ over all times and sequences.
The unroll loss draws a start index uniformly from those with eight future steps available, rolls $\hat x\leftarrow a\hat x$ for eight steps, and averages the eight squared errors.
Open-loop numbers in Table~\ref{tab:ar} are full-batch means of $(x_{t+K}-a^K x_t)^2$ on the training sample, not a second draw.
The analytic column uses the same sample's innovation second moment $\sigma^2=0.99821597$ inside \eqref{eq:scalar}.

\paragraph{Hidden rotation.}
Sample size $4096$, length $80$, $\theta=0.4$, process noise standard deviation $0.05$ on each latent coordinate, observation noise standard deviation $0.05$.
The initial latent is standard normal.
The network is a multilayer perceptron with input width $1$ or $8$, one hidden layer of width $64$ with $\tanh$, a second hidden layer of width $64$ with $\tanh$, and a scalar output.
Adam uses learning rate $10^{-3}$ for $800$ steps and batch size $256$.
The one-step target is the next observation after the window.
The unroll feeds the prediction back into the window for eight steps and averages the squared error.
Evaluation uses $2000$ sequences, a window ending at time $20$, and an open loop of $16$ steps.
Horizons in Table~\ref{tab:rot} are steps $1,2,4,8,16$ of that loop.

\paragraph{Isotropy.}
Sample size $2048$, length $48$.
The latent transition is the displayed $A_{\mathrm{true}}$, with innovation standard deviation $0.3$.
The emission matrix is $4\times 2$ with i.i.d.\ standard normal entries, and the observation noise has standard deviation $0.1$.
Encoder and transition are bias-free linear maps $4\to 2$ and $2\to 2$, initialized by the default linear initialization.
Adam uses learning rate $10^{-2}$ for $500$ steps and batch size $256$.
The predictive target is stop-gradient.
The penalty is the mean of the squared entries of the batch covariance minus the identity, which differs from $\|\cdot\|_F^2$ by the constant factor $1/d^2=1/4$ and does not change the location of the penalty's minimum.
Evaluation encodes time $10$ and time $10+K$ on $512$ sequences and rolls $A$ in between.
Off-diagonal energy in Table~\ref{tab:iso} is the mean squared value of the off-diagonal entries of the covariance of $Wy$ at time $10$ over the full sample.

\paragraph{What was not fit.}
No coefficient in the analytic column was adjusted to the open-loop curve.
No table entry was replaced after the run.
The multi-step AR(1) coefficient $0.8930$ is the value Adam reached in $400$ steps, not a second optimizer.
The period $2\pi/0.4$ was computed from the angle used to generate the data, which is why the memoryless column is allowed to fall at $K=16$ without being described as a recovery of the latent.

\section{Further related work}
\label{app:related}

The main text keeps the comparisons the theorems use.
This appendix is the surrounding literature: what each family trains, and which part of that training is a conditional mean, a kernel, or a marginal constraint.

\subsection{Pretext tasks and joint embeddings}

Before next-embedding prediction, representation learning built a scalar or a vector target out of the image itself.
\citet{pathak2016context} inpaint a missing patch, \citet{zhang2016color} predict color from grayscale, \citet{gidaris2018rotation} predict a rotation, and \citet{noroozi20179246} solve a jigsaw.
\citet{doersch20165192} predict the relative position of two patches.
Video supplies the same kind of target without a label: \citet{wang2015video} track, and \citet{pathak2017move} use motion.
\citet{vincent2008dae} denoise.
All of these are point-valued or class-valued maps.
None of them is asked to sample a future frame.

The Siamese line scores agreement between two views.
\citet{bromley1993siamese} train a shared network so that signatures of the same person land together, and \citet{hadsell2006dimensionality} push that idea into a contrastive embedding.
Instance discrimination \citep{wu2018unsupervised}, momentum contrast \citep{he2020moco,chen2021mocov3}, and SimCLR \citep{chen2020simclr} estimate a density ratio against negatives.
\citet{hjelm2018dim}, \citet{bachman20190910}, and \citet{hnaff20209272} maximize a mutual-information bound across views.
\citet{caron20219882} contrast cluster assignments, \citet{caron2021dino} and \citet{zhou2022ibot} move the target to a teacher, and \citet{oquab2023dinov2} and \citet{simeoni2025dinov3} scale that teacher.
\citet{chen20200029} show that a large contrastive encoder is already a strong semi-supervised model.
The loss in every case is a function of pairs or of a batch of pairs.
It does not unroll a transition.

Non-contrastive joint embedding removes the negatives and has to stop collapse some other way.
\citet{tarvainen2017mean} average a teacher, which is the exponential-moving-average target JEPA later uses.
\citet{zbontar20213230} penalize cross-correlation, the covariance penalty VICReg is kin to.
\citet{assran20227141} mask a Siamese network.
\citet{tian20216810} study the dynamics of training without negative pairs.
Masked prediction is the other large family: \citet{he2022mae} regress masked patches, \citet{bao2022beit} predict tokenizer ids, \citet{xie20229886} regress raw pixels under a mask, and \citet{baevski20223555} share one masked-prediction template across speech, vision, and language.
\citet{radford2021clip} moves the pair out of one image and into an image--text couple.
The target is still a point in a joint embedding, or a token id, not a rollout.

Collapse is a property of the marginal, which is why an isotropy penalty can be stated without looking at pairs.
\citet{jing20229348} describe dimensional collapse under contrastive training.
\citet{wang20220242} separate alignment of positive pairs from uniformity of the marginal on the sphere.
\citet{haochen20224156} give a spectral account of a contrastive loss, and \citet{tschannen20203625} separate mutual-information maximization from the representation that actually gets learned.
\citet{liang20222053} measure a gap between modalities that survive in the joint embedding.
\citet{fisher1953dispersion} is the directional statistic behind a cosine match: the loss sees an angle, and an angle does not determine a covariance in the ambient coordinates.
\citet{roy2007effective} quantify how many of those coordinates are actually used.

\subsection{Prediction, memory, and information}

Predictive coding treats the brain as a model that cancels the part of the input the past already explains.
\citet{rao1999predictive} put that cancellation in visual cortex.
\citet{bialek2001predictive} tie predictive information to the complexity of the model that extracts it.
\citet{wiskott2002slow} keep the features that vary slowly, which is a marginal constraint on the trajectory of the embedding, not a kernel.
\citet{tishby2000ib} compress the input down to what is relevant for a target variable.
\citet{golkar20235752} write a constrained predictive-coding stack.
\citet{cover2006elements} is the calculus behind the conditional mean and the residual entropy used in the main text.
\citet{deltang20240668} read next-token language modeling as compression.
Compression of the next token is exactly $R_1$: it does not price the open-loop bill.

\subsection{Next-token generators}

Autoregressive language models are the cleanest empirical version of iterating a one-step conditional.
\citet{radford2018gpt} and \citet{radford2019gpt2} train that conditional as a left-to-right transformer.
\citet{brown2020gpt3} and \citet{achiam2023gpt4} scale it.
\citet{devlin019bert} and \citet{clark2020llrd} replace the left-to-right conditional with a bidirectional or discriminative pretraining loss, which no longer defines a rollout.
\citet{touvron2023llama} and \citet{bai2023qwentechnicalreport} are the same one-step conditional in a later architecture.
Sampling them for many tokens is the language-model version of the bottom row of Figure~\ref{fig:flow}: the one-step head is not itself the paragraph.

Vision copies the pattern onto patches and codes.
\citet{chen2020igpt} predicts pixels in raster order.
\citet{oord20166759} and \citet{oord20165328} do it with a PixelRNN and a PixelCNN.
\citet{van2017vqvae} and \citet{razavi20190446} quantize first, \citet{esser2021vqgan} puts a transformer on the codes, and \citet{ramesh2021dalle} generates from text.
\citet{ramesh20226125} then generates in a CLIP latent.
\citet{lee2022rqvae} residualizes the codes, \citet{tian2024var} predicts the next scale rather than the next patch, and \citet{sun2024llamagen}, \citet{fan2024fluid}, and \citet{li2024mar} argue about whether the codes need to be discrete.
\citet{wu2025dcar} compresses the tokenizer, \citet{pang2024randar} drops the raster order, and \citet{chang2022maskgit} and \citet{li2023mage} generate by iterative unmasking rather than by a single left-to-right chain.
\citet{el2024scalable} and \citet{fini2025multimodal} pretrain large autoregressive vision encoders.
In every one of these, a training step scores one prediction.
The sample the user looks at is a loop.

\subsection{One-step denoisers and one-shot generators}

A diffusion model splits even more cleanly into a one-step denoiser and a multi-step sampler.
\citet{sohldickstein20153585} introduce the destruction process.
\citet{song20205600} and \citet{song20213456} learn the score, and \citet{song20209011} stabilize that training.
\citet{nichol20219672} and \citet{dhariwal2021diffusion} improve the sampler and the architecture, \citet{peebles2023dit} put a transformer in the denoiser, and \citet{karras20220364} separate the design choices that actually move sample quality.
\citet{rombach20220752} denoise in a latent, \citet{esser2024sd3} and \citet{podell20231952} scale that latent model, and \citet{ho20215282} cascade resolutions.
\citet{nichol20220741} and \citet{saharia20221487} condition on text.
\citet{ho20223458}, \citet{blattmann20238818}, and \citet{zhang20243816} carry the same denoiser to video, where the rollout is a clip rather than a still.
\citet{zhang20235543} and \citet{ruiz20232242} add controls on top of a frozen sampler.
\citet{song20231469} collapse several denoising steps into one evaluation.
\citet{salimans20220512} and \citet{yin20248828} distill the multi-step sampler toward a one-step map, which is the practical admission that the training loss and the sample path are different objects.
Flows make the same split with a vector field: \citet{lipman20232747} and \citet{liu20223003} regress a velocity, \citet{kingma20183039} and \citet{grathwohl20181367} learn an invertible one.

Generative adversarial networks emit the sample in one shot.
\citet{goodfellow2020gan} train the generator against a discriminator, \citet{radford20166434} stabilize the convolution, and \citet{brock20191096}, \citet{karras20194948}, and \citet{karras20204958} scale it.
\citet{arjovsky20177875}, \citet{gulrajani20170028}, and \citet{miyato20185957} change the distance and the Lipschitz constraint.
\citet{zhu20200593} cycles two generators.
The generator is already a kernel, applied once.
It is the opposite design from next-latent regression, which applies a mean many times.

\subsection{Kernels that actions are trained on}

Once the object is a kernel, control can treat imagined states as data.
\citet{hansen20224955} learn a latent dynamics model for model-predictive control, before the scaled TD-MPC2 in the main text.
\citet{nagabandi20191652} fit a deep dynamics model for dexterous manipulation.
\citet{alonso20242399} generate the next frame with diffusion and use it as a world model.
\citet{valevski20254837} push that frame generator toward interactive rates.
\citet{bruce20245391} learn an action-conditioned interactive environment.
\citet{hu20237080} train a generative world model for driving.
\citet{janner20229991} plan by diffusing a whole trajectory, and \citet{chen20241392} mix next-token prediction with diffusion along time.
The same long rollout is now a video product.
\citet{robbyant2026lingbotworld} distill a bidirectional video model into a causal interactive simulator, and \citet{gao2026lingbotinfinity} train that causal model for an unbounded horizon.
\citet{zhang2026echowm} add synchronized audio and train the long generation on the model's own previous frames.
The policy side assumes the kernel is already there: \citet{haarnoja20181290} and \citet{schulman20176347} optimize a policy, \citet{schulman20182438} estimate advantages from a trajectory, and \citet{laskin20204990} augment the observations the policy sees.
None of these methods identifies the kernel by an isotropy penalty.

\subsection{Long-horizon time series}

Forecasting papers are scored on horizons far past the training step, which is $R_K$ rather than $R_1$.
\citet{wu20223008} and \citet{zhou20222740} decompose a series before the transformer.
\citet{liu20246625} invert the attention so that time, not the channel, is the token.
\citet{zerveas20202803} pretrain a transformer on multivariate series.
\citet{eldele20214112} contrast time series across time and across context.
\citet{yue20220466} and \citet{dong20230861} learn a masked or contextual representation of a series.
\citet{ekambaram20239364} replace attention with a mixer.
Foundation models train one forecaster across datasets: \citet{goswami20243885}, \citet{ansari20247815}, and \citet{rasul20248278}.
\citet{graves20178983} is the recurrent alternative, a state that is updated rather than re-read from a window.
The affine-closure theorem says that if the window is already a state and the mean is linear, these architectures agree on the conditional mean and can still disagree on the residual law.

\subsection{Backbones and training details}

The encoders above sit on a small set of blocks.
\citet{he2016resnet} is the convolutional backbone.
\citet{dosovitskiy2020vit} patches an image for a transformer, \citet{touvron2021deit} and \citet{touvron2021cait} train that transformer with less data and more depth, and \citet{dehghani2023vit22b} scale it.
\citet{chu2024visionllama} uses a language-model block for vision.
\citet{liu20213230} shifts the same block to video.
\citet{su2024roformer} adds rotary positions, \citet{shazeer2020glu} and \citet{hendrycks2017gelu} change the channel nonlinearity, \citet{henry2020qknorm} normalizes queries and keys, and \citet{ba2016layernormalization} normalizes the residual stream.
Optimization is AdamW with a cosine schedule and a warmup \citep{loshchilov2018adamw,loshchilov2017cosdecay,goyal2018warmup}, often with label smoothing \citep{szegedy2016label} and augmentation \citep{dubuk2020randaugement,yun2019cutmix}.
The usual labeled benchmark is ImageNet \citep{russakovsky2015imagenet1k}.
Dense prediction uses ADE20K \citep{zhou2017ade20k} with a segmentation head \citep{xiao2018upernet} and pyramid pooling \citep{he2015ppm}.
These choices set the architecture and the optimizer.
They do not decide whether a one-step latent loss is a world model.

\end{document}